\documentclass[letterpaper]{article} 
\usepackage{aaai24}  
\usepackage{times}  
\usepackage{helvet}  
\usepackage{courier}  
\usepackage[hyphens]{url}  
\usepackage{graphicx} 
\usepackage{natbib}  
\usepackage{caption} 
\usepackage{algorithm}
\usepackage{algorithmic}
\usepackage{newfloat}
\usepackage{listings}
\DeclareCaptionStyle{ruled}{labelfont=normalfont,labelsep=colon,strut=off} 
\floatstyle{ruled}
\newfloat{listing}{tb}{lst}{}
\floatname{listing}{Listing}
\usepackage{multirow}
\usepackage{booktabs}
\usepackage{tikz}
\usepackage{epsfig}
\usepackage{graphicx}
\usepackage{amsmath}
\usepackage{amssymb}
\usepackage{microtype}
\usepackage{graphicx}
\usepackage{subfigure}
\usepackage{booktabs,bm} 
\usepackage{xspace}
\usepackage{amsthm}
\usepackage{tikz}
\usepackage{tikz-cd}
\usepackage{pgfplots}
\pgfplotsset{compat=newest}
\usepackage{physics}
\usepackage{lilyglyphs}
\usepackage{graphicx}

\newtheorem{proposition}{Proposition}
\def\method{LION}

\title{\method{}: Implicit Vision Prompt Tuning}

\title{\method{}: Implicit Vision Prompt Tuning}
\author {
Haixin Wang\textsuperscript{\rm 1},
Jianlong Chang\textsuperscript{\rm 2, $\dagger$},
Yihang Zhai\textsuperscript{\rm 1},
Xiao Luo\textsuperscript{\rm 3},\\
Jinan Sun\textsuperscript{\rm 1, $\dagger$}, 
Zhouchen Lin\textsuperscript{\rm 4, 5},
Qi Tian\textsuperscript{\rm 2}
}
\affiliations {
    \textsuperscript{\rm 1}National Engineering Research Center for Software Engineering, Peking University,\\
    \textsuperscript{\rm 2}Huawei Cloud \& AI, \\
    \textsuperscript{\rm 3}School of Mathematical Sciences, Peking University, \\
    \textsuperscript{\rm 4} National Key Lab of General AI, School of Intelligence Science and Technology, Peking University, \\
\textsuperscript{\rm 5} {Peng Cheng Laboratory} \\
    wang.hx@stu.pku.edu.cn, jianlong.chang@huawei.com, zhaiyihang@stu.pku.edu.cn,\\
 \{xiaoluo, sjn, zlin\}@pku.edu.cn, tian.qi1@huawei.com
}

\usepackage{bibentry}

\begin{document}

\maketitle

\begin{abstract}
Despite recent promising performances across a range of vision tasks, vision Transformers still have an issue of high computational costs.
Recently, vision prompt learning has provided an economical solution to this problem without fine-tuning the whole large-scale model. 
However, the efficiency and effectiveness of existing models are still far from satisfactory due to the parameter cost of extensive prompt blocks and tricky prompt framework designs. 
In this paper, we propose a light-weight prompt framework named imp\underline{L}icit v\underline{I}sion pr\underline{O}mpt tu\underline{N}ing (\textbf{\method{}}), which is motivated by deep implicit models with stable low memory costs for various complex tasks.
In particular, we merely insert two equilibrium implicit layers in two ends of the pre-trained backbone with parameters frozen. Moreover, according to the lottery hypothesis, we further prune the parameters to relieve the computation burden in implicit layers. Various experiments have validated that our \method{} obtains promising performances on a wide range of datasets. Most importantly, \method{} reduces up to 11.5 \% of training parameter numbers while obtaining higher performance than the state-of-the-art VPT, especially under challenging scenes. Furthermore, we find that our proposed \method{} has an excellent generalization performance, making it an easy way to boost transfer learning in the future.
\end{abstract}

\section{Introduction}

With the development of computer vision, models with more robust representations and larger sizes have been developed. Despite this, training these models with many parameters is becoming increasingly challenging.

One common approach to addressing this issue is pre-training on a large dataset, such as ImageNet~\cite{deng2009imagenet}, for general vision tasks and then fine-tuning the model on downstream tasks to improve performance. While this method has been widely used, several drawbacks should be considered. Firstly, fine-tuning requires a large amount of computational resources, especially for large models such as ViT-B~\cite{dosovitskiy2020image} (85.84M parameters) and Swin-B~\cite{liu2021swin} (86.87M parameters).  
Secondly, the model may become overfitted to the small target dataset and cannot be used for other tasks after fine-tuning. This phenomenon means separate sets of model parameters are needed for each task, leading to a high storage requirement.

In recent years, prompt-based learning has generated considerable interest in the natural language processing (NLP) community because of its fantastic performance on various downstream problems~\cite{liu2021pre,li2021prefix,lester2021power}. 
Prompt tuning aims to design a trainable lightweight block as a supplementary input, which can guide or direct the generation of powerful vision representations to achieve desirable performances, rather than fine-tuning pre-trained models to adapt to downstream tasks. 

\begin{figure}[t]
    \centering
    \includegraphics[width=0.43\textwidth]{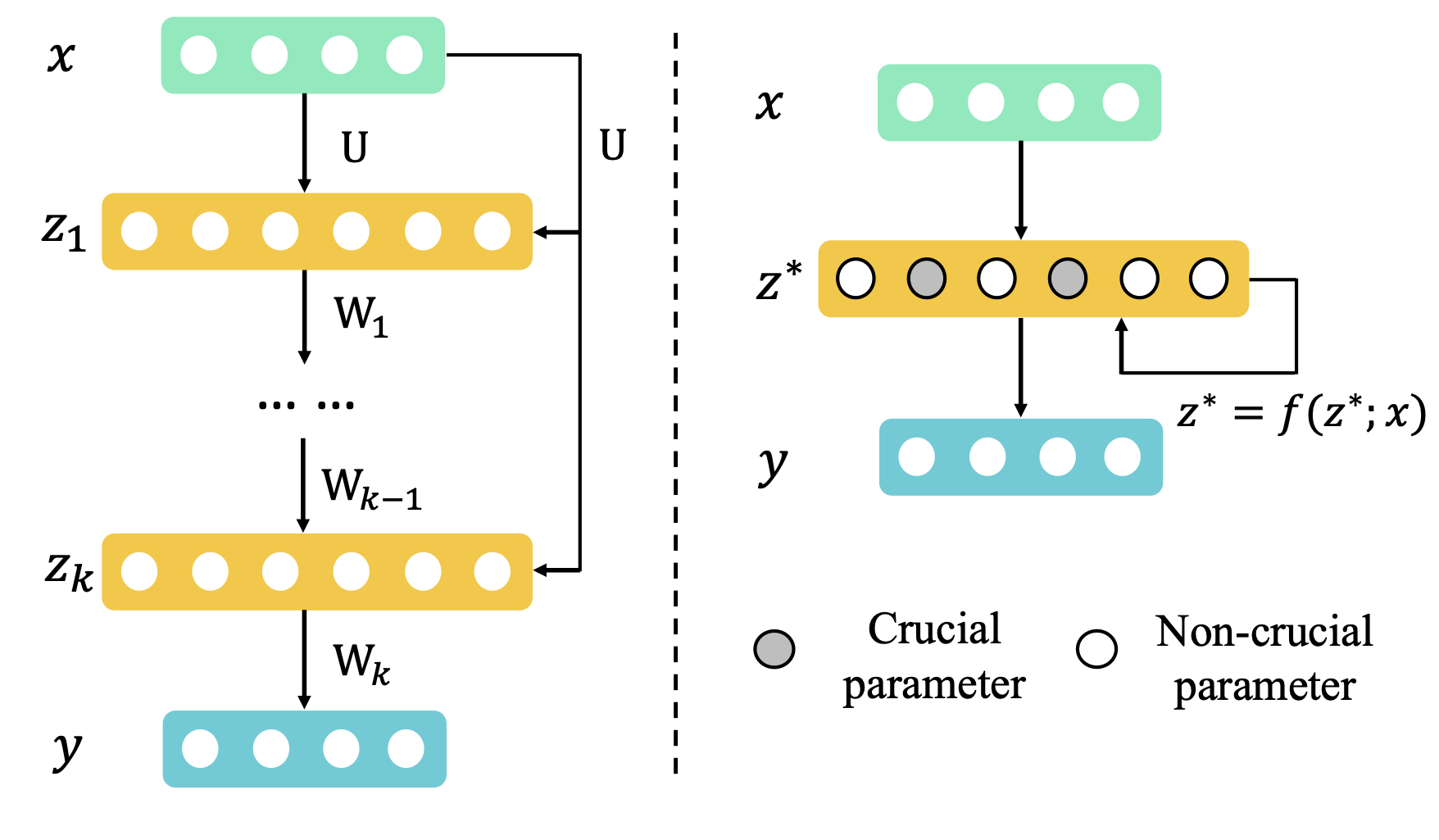}
    \caption{Demonstration of the implicit vision prompt layer. The left part shows the traditional construction of the prompt block by stacking MLPs. The right is \method{} with the implicit equilibrium layer and robust training for the prompt block.  }
    \label{fig:intro}
\end{figure}

To design a prompt framework that combines both a lightweight architecture and strong representation capability, we conducted a comprehensive study and analysis of the limitations of current vision prompt tuning methods.
\textbf{First}, existing approaches insert trainable networks as prompt blocks between each layer of the network \cite{jia2022visual}, assuming that the feature representations from different levels contribute to the network's generalization performance, especially for low- and mid-level representations. This, however, goes against the lightweight design philosophy of prompt tuning. The architecture design is also complex and heavily reliant on tuning skills, making applying to various vision backbone models with different architectures challenging.
\textbf{Second}, finding the right prompt is a challenging task that often takes a significant amount of time. Even small changes in the activation input can significantly impact performance. This can be attributed to the depth of big model architectures, making the trainable parameters of shallow network layers more challenging to train and converge.

Based on the challenges above, we naturally raise a question, \textit{\textbf{can we design a single-layer network as the prompt block with favorable convergence to iterate continuously?}} We hope it can achieve the effect of multi-layer network training and thus significantly reduce the training parameters. 
Therefore, we propose the imp\underline{L}icit v\underline{I}sion pr\underline{O}mpt tu\underline{N}ing (\textbf{\method{}}), which is motivated by deep implicit models with stable low memory costs for various complex tasks. In particular, we merely insert two equilibrium implicit layers in two ends of the pre-trained backbone with parameters frozen. \method{} enables tuning various vision models, including convolutional neural networks (CNNs) and vision transformers.  

Specifically, \method{} constructs a lightweight prompt framework to generate task-specific discriminative prompts for each downstream input image. \method{} can generate a compact and robust downstream model that adapts tuning demands across a wide range of vision tasks while only training lightweight additional parameters per downstream task, which is implemented by blending the input and the representations with the learned prompts. 
Besides, since the hyper-parameters, like the learning rate, can severely affect the robustness of training the vision prompts, we prune the parameters in these two layers according to the lottery hypothesis. Only the critical parameters are kept in training to avoid over-fitting. 
More surprising is that \method{} essentially compresses the parameters of the existing vision prompt network, which allows it to be generalized to any subsequent vision prompt tuning method with trainable parameters.

Our proposed \method{} can be used to tune CNN-based and Transformer-based vision models, surpassing fine-tuning for various recognition tasks of image classification. It can perform well under a variety of practical scenarios, including generic objects, class imbalance, and few-shot learning. 
Learning vision-specific information while maintaining the pretrained knowledge, our \method{} delivers an average improvement of up to 1.3\% compared with VPT, but with much fewer trainable parameters. In summary, the main contributions of our work are three-fold:
\begin{itemize}
    \item We propose \method{}, a significantly lightweight yet effective tuning method that leverages few trainable implicit layers to adapt the pretrained model to downstream visual tasks.
    \item We construct a more robust optimization strategy by lottery hypothesis while proving the excellent convergence of our method through theoretical analysis and validation.
    \item Experimental results show that \method{} outperforms previous vision prompt tuning methods, evidencing its feasibility and usefulness for vision models, especially on few-shot and long-tail scenarios.
\end{itemize}
 \section{Related Work}

\begin{figure}[htbp]
    \centering
    \includegraphics[width=0.46\textwidth]{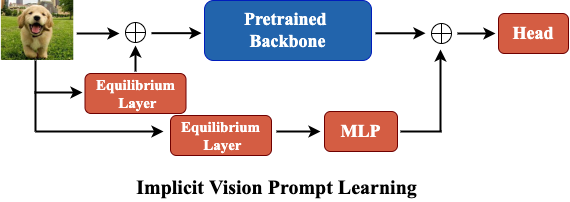}
    \caption{Structures of our \method{}. We add two implicit layers, which are only injected in front of the input and behind the output of the pre-trained backbone respectively, as the vision prompts to enrich the vision input and representation. }
    \label{fig:framework}
\end{figure}

\subsection{Vision Fine-tuning}
Recently, more and more works fine-tune the pre-trained models as the backbone for downstream tasks \cite{zaken2022bitfit,croce2022adversarial,wortsman2022model,yu2023visual}, which are trained on large-scale image datasets for common tasks like image classification.
Fine-tuning is highly flexible: it can be applied to new input domains or tasks with different output semantics.

Some work has focused on developing fine-tuning methods that allow for the adaptation of the entire network~\cite{houlsby2019parameter,pfeiffer2020adapterfusion,pfeiffer2020adapterhub}, rather than just a subset of layers.
Another line of related work has focused on developing fine-tuning methods that can be used in a transfer learning setting~\cite{ijcai2022p769,sung2022vl,lin2020exploring,houlsby2019parameter}, where the pre-trained model is adapted to a new task in a different domain.

\subsection{Prompt-based Learning}
Prompt-based learning~\cite{liu2021pre} is a technique that utilizes task-specific descriptions to enhance the understanding of downstream tasks by pre-trained models. This approach was popularized by the GPT series~\cite{brown2020language,radford2018improving,radford2019language} in the field of NLP. This has led to many studies focused on developing effective prompt strategies for extracting knowledge from pre-trained language models. Similarly, recent vision-language models~\cite{lester2021power,gao2020making,li2021prefix,schick2020exploiting,yang2022diffusion,yang2023improving,yu2023visual,wang2023parameter} have achieved impressive performance on various vision tasks without the need for fine-tuning. However, these prompt-based tuning methods are not suitable for pre-trained vision models. Our work aims to bridge this gap by developing a parameter-efficient prompt tuning approach specifically for vision models, to adapt frozen pre-trained vision models to downstream tasks across a broad distribution.

\subsection{Deep Implicit Models}
There has been a growing interest in using implicit layers in deep learning. Researchers have explored different approaches to utilizing numerical analysis methods to replace the representation mechanism in existing deep networks. Some notable examples include SparseMAP \cite{niculae2018sparsemap}, OptNet \cite{amos2017optnet}, and SATNet \cite{wang2019satnet}. These approaches have shown promising results in improving the efficiency and performance of deep learning models.
One particular type of implicit model that has gained attention is Deep Equilibrium Models (DEQ) \cite{bai2019deep}. DEQ is an implicit model with infinite depth, yet it is interesting as a single-layer network because it allows for analytical backpropagation through the equilibrium point. Regardless of the depth of the network, training and predicting with DEQ only require constant memory. Moreover, DEQ has achieved comparable performance with efficient memory cost, as illustrated in \cite{xie2021optimization}.
Another advantage of DEQ is its interpretability. The use of implicit models can make it difficult to interpret the behavior of the network, but DEQ provides a transparent mechanism for understanding the network's inner workings. These advantages make DEQ a suitable candidate for constructing lightweight vision prompt layers.

\section{Methodology}

\subsection{Overall Architecture}
In vision prompt tuning, the goal is to adapt a pre-trained vision model to downstream tasks without modifying its weights and achieve comparable results with the commonly used fine-tuning method. 
Mathematically, given a pre-trained large-scale vision model $\mathcal{G}$ with parameters $\Theta$, it can be decomposed into two parts, the backbone $\mathcal{F}$ with frozen parameters $\Theta_f$ and the head layer $\mathcal{H}$ with trainable parameters $\Theta_h$. Thus, the input image $x_i \in \mathbb{R}^{H \times W \times 3}$ from the down-stream dataset $\mathcal{D}=\{(x_i,y_i)\}_{i=1}^N$.
Our method, named \method{}, accomplishes the goal with an extremely lightweight prompt block $\mathcal{P}$ with a few trainable parameters $\Theta_p$, which utilizes the implicit equilibrium layer for activation. We only insert two prompt blocks in front of the input and head layers, respectively. Suppose that the output of the backbone is $z_i = \mathcal{F}(x_i;\Theta_f)$, two prompt-based blending representations can be written as:
\begin{equation}
\label{eq:p1}
    \tilde{x}_i = \alpha_1 x_i + \beta_{1} \mathcal{P}_{1}(x_i;\Theta_{p1}),
\end{equation}
\begin{equation}
\label{eq:p2}
\tilde{z}_i = \alpha_{2} \mathcal{F}(\tilde{x}_i;\Theta_f) + \beta_{2} MLP\left(\mathcal{P}_{2}(x_i;\Theta_{p2})\right).
\end{equation}
Here, $MLP$ represents the fully-connected layer for representation projection. $\alpha_i$ and $\beta_i$ represent balance coefficients determined based on their importance in the attention mechanism~\cite{mnih2014recurrent}. To generate these balance coefficients, we initialize two adjustable parameters, $g_{{\alpha}_i}$ and $g_{{\beta}_i}$, subject to the following constraints: ${\alpha}_i + {\beta}_i = 1$ and $0 < {\alpha}_i, {\beta}_i < 1$.
\begin{equation}
    \alpha_i = \frac{e^{g_{\alpha_i}}}{e^{g_{\alpha_i}}+e^{g_{\beta_i}}},  \beta_i = \frac{e^{g_{\beta_i}}}{e^{g_{\alpha_i}}+e^{g_{\beta_i}}}
\end{equation}
In this way, the task-specific knowledge of the downstream data is distilled and effectively incorporated into the trainable parameter set: $\Theta_{t} = Concat(\Theta_{p1} || \Theta_{p2} || \Theta_{h})$, activating the pre-trained model. 
The prompt-based blending representation promotes a balance between the original representations and the learned prompt, resulting in adaptive control. Our prompt blocks perform a generic architectural modification that enables the pre-trained vision model to be adapted to a downstream task with only a few additional parameters $\Theta_{t}$.

Intuitively, \method{} is illustrated in Figure \ref{fig:framework}. The whole process can be listed as follows: \textbf{i)} Blend the downstream input image with the vision prompt using an adaptive coefficient by feeding the downstream input image into the equilibrium layer. \textbf{ii)} Feed the combined image to the frozen pre-trained model to get the feature representations. \textbf{iii)} Input the downstream image into the equilibrium layer to produce vision prompts, and map the prompts to the representation size to create a combination. \textbf{iv)} Add a fully-connected layer to the top layers of the pre-trained model for the final prediction.

\subsection{Implicit Prompt Design}
The Deep Equilibrium (DEQ) Model, as described in \cite{bai2019deep}, employs a single layer that finds the fixed point of an iterative procedure. This layer, capable of expressing the entire deep network as an equilibrium computation, is just as powerful as multiple stacked explicit layers. Our proposed method, \method{}, leverages this capability by using a single DEQ layer to adapt a pre-trained vision model to downstream tasks through implicit layer training. The downstream inputs are equipped with a single implicit layer, implemented as a ResNet layer~\cite{he2016deep}, which is trained to learn task-specific prompts while the pre-trained model remains frozen.
Drawing inspiration from previous studies \cite{islam2021broad,lin2017refinenet,yosinski2014transferable}, we aim to improve the feature representation by utilizing high-frequency and low-frequency information as the vision prompt. To attain this goal while reducing the parameter storage burden, we propose using two-level prompt blocks, located prior to the input layer and the head layer. Our design of this lightweight architecture is supported by demonstrations of its convergence ability in Theoretical Analysis.

Unlike a conventional network where the output is the activation from the $L$ th layer, the output of an equilibrium layer is the equilibrium point itself. Therefore, the forward evaluation could be any procedure that solves for this equilibrium point. Conventional deep neural networks, if they converge to equilibrium, can be considered a form of fixed-point iterations for the forward process:
\begin{equation}
    z^{*} = \mathcal{P}(z^{*}, x; \Theta_p)
\end{equation}
Our goal will be to compute the vector-Jacobian product $\frac{\partial z^{*}(\cdot)}{\partial(\cdot)}^T y$ for some vector $y$, where $(\cdot)$ here is a stand-in for any quantity we want to differentiate the fixed point with respect to (i.e., the input $x$, or any parameters of the function $\mathcal{P}$, both of which of course will affect the final fixed point $z^*$). Since this vector-Jacobian product is the key aspect to integrating these DEQ layers within backpropagation, such a routine allows us to integrate the DEQ layer within standard automatic differentiation tools.

The derivation of the vector-Jacobian product largely mirrors that in previous sections, but we include the full derivation again here for completeness. Differentiating both sides of the fixed point solution, we have:
\begin{equation}
    \frac{\partial z^{*}(\cdot)}{\partial(\cdot)} = \left(I - \frac{\partial \mathcal{P}(z^{*},x)}{\partial(z^*)}\right)^{-1} \frac{\partial \mathcal{P}(z^{*},x)}{\partial(\cdot)}
\end{equation}

In order to calculate the vector-Jacobian product, we need the following information:
\begin{equation}
    \left(\frac{\partial z^{*}(\cdot)}{\partial(\cdot)}\right)^T y = \left(\frac{\partial \mathcal{P}(z^{*},x)}{\partial(\cdot)}\right)^T \left(I - \frac{\partial \mathcal{P}(z^{*},x)}{\partial(z^*)}\right)^{-T}y 
\end{equation}

The critical term of interest here is the solution in a linear system, and we can utilize the proxy $o = (I - \frac{\partial \mathcal{P}(z^{*},x)}{\partial(z^*)})^{-T}y$ to solve the fixed point equation and compute the final Jacobian vector product.

\subsection{Robust Training}

We utilize the robust training mechanism for trainable parameters $\Theta_t$ motivated by \cite{frankle2018lottery} to overcome the instability during prompt tuning. The Lottery Ticket Hypothesis in it posits that only a subset of a model's parameters is crucial for achieving good generalization, the rest are likely to overfit. To address this, we employ a criterion to separate crucial and non-crucial parameters and optimize them differently. Crucial parameters are updated with stochastic gradient descent, while non-crucial parameters are constrained to reduce their ability to overfit.

Recent pruning methods~\cite{frankle2018lottery,yeom2021pruning,xia2020robust} suggest that crucial parameters should have substantial magnitudes, as they play a key role in network propagation. Furthermore, early optimization shows that parameters with large gradients tend to contribute to generalized patterns~\cite{molchanov2019importance}, which are crucial for learning from clean samples. Therefore, the parameters' values and gradients should be considered when determining their importance. To capture this, we use the product of their values and gradients as the criterion for determining criticality. In mathematical terms, the criticality of parameter $\theta_t$ from $ \Theta_t = \{\theta_t^i \}_{i=1}^M$ is represented as:

\begin{equation}
	z(\theta_t) = \left|\frac{\partial \mathcal{L}}{\partial{\theta_t}}\cdot \theta_t \right|,
\end{equation}
where $\mathcal{L}$ represents the loss function.

The equation shows that when the gradient or value of a parameter is close to zero, its criticality is low, making it a non-crucial parameter that is prone to overfit. On the other hand, when the value of $z(\theta_t)$ is immense, $\theta_t$ is considered a crucial parameter for learning basic and generalized patterns. To control the number of crucial parameters, we introduce a threshold $\tau $, such that crucial parameters are selected and represented as:

\begin{equation}\label{eq:phi_c}
	\Theta_t^c = \{\theta_t | z(\theta_t) \geq \tau   \},
\end{equation}
\begin{equation}\label{eq:phi_n}
	\Theta_t^n = \{\theta_t | z(\theta_t) < \tau    \}.
\end{equation}

Unlike pruning methods, we do not eliminate the non-crucial parameters $\Theta_t^n$. Instead, we adopt a distinct optimization strategy. Here, the crucial parameters are updated in the usual way, while the non-crucial parameters are restricted to converge to zero for better generalization. Mathematically, the update rule for $\theta_t \in \Theta_t^c$ is represented as:

\begin{equation}\label{eq:13}
	\theta_t \leftarrow \theta_t - \eta \frac{\partial \mathcal{L}}{\partial{\theta_t} },
\end{equation}

The symbol $\eta$ denotes the learning rate in the above equation. For the non-crucial parameters, we shrink them utilizing strict regularization instead of minimizing the loss. In other words, the update rule for $\theta_t \in \Theta_t^n$ is:

\begin{equation}\label{eq:14}
	\theta_t \leftarrow \theta_t - \eta sign(\theta_t).
\end{equation}

\subsection{Optimization}

For our training process, we keep the pre-trained parameters $\Theta_f$ intact and only modify a limited set of parameters $\Theta_t$. This selective update of parameters makes our \method{} modular and efficient - it allows us to utilize an existing pre-trained vision model without having to modify or re-train it. Instead, we add a small number of additional parameters specific to each task,  which can be formulated as:
\begin{equation}
    \theta_t^* = \mathop{argmin}\limits_{\theta_t} \frac{1}{|\mathcal{D}|} \sum_{i=1}^N \ell \left( \mathcal{H}(\tilde{z}_i), y_i \right)
\end{equation}
With the pre-trained model frozen, we minimize the prediction error using cross-entropy (CE) loss. The ability of our \method{} to adapt a pre-trained vision model to a wide range of tasks while maintaining a high level of accuracy makes it a desirable solution for deployment in cloud services. The potential benefits of reduced computational and storage overhead, as well as the ability to offer real-time adaptation to new tasks, make our method an ideal choice for cloud service providers seeking to improve their offerings.

\subsection{Theoretical Analysis}
\label{sec:theory}

\noindent\textbf{Theoretical setup.} Our proposed \method{} aims to provide prompts to the vision input $x \in \mathbb{R}^d$ and the representation $z \in \mathbb{R}^h$ derived by the backbone. 
The whole network maps $x$ to the label $y \in \mathbb{R}$ with the loss $l(y,\hat{y})$ like cross-entropy, etc. 
We consider the network to be $f_{\theta}(x) = v^{\mathrm{T}}\sigma(Wx)$, where $v \in \mathbb{R}^k$, $W \in \mathbb{R}^{k \times d}$, and $\sigma$ is an element-wise activation function like ReLU. 
Assume that the random variables $x, y \sim P$ and the population loss can be denoted as 
$\mathcal{L}(\theta) = \mathbb{E}\left[l(f_{\theta}(x),y)\right] $.

Note that we have only two prompt blocks in different positions and architectures. We show that we can directly add the vision prompts to the first layer of the pre-trained model, while it cannot be implemented on the last layer. We should utilize another MLP block to ensure the optimal solution. Here we assume the $\sigma$ to be ReLU: $\sigma(x_i) = max(x_i,0)$. Given that the model is pre-trained with parameters $\hat{\theta} :(\hat{v}_{pre}, \hat{W}_{pre})$ which reach the optimal solution $\mathcal{L}(\hat{v},\hat{W}) = 0$. With the single-layer DEQ as the vision prompts network, we can derive the vision prompts with the suppose that: $x_{pro} = A x$ and $z_{pro} = B z$ for some invertible matrix $A, B$, where the corresponding label is unchanged: $y_{pro} = y$.

\begin{proposition}
   There exists the vision prompt $x_{pro} = A x$ for invertible $A$ and $y_{pro} = y$ that can minimize the population loss: $min_{W}\mathcal{L}(\hat{v}, W) = 0$. However, the vision prompt $z_{pro} = B z$ may not be sufficient: there exists such $B$ such that the population loss is non-zero for any choice of the parameter $v$: $min_{v}\mathcal{L}(v, \hat{W}) > 0$.
\end{proposition}
\begin{proof}
    Let $\hat{B}, \hat{v}$ can reach the optimal solutions so that $y = \hat{v}\sigma(\hat{W}x)$ for all $x,y$. Let $W = \hat{W} A^{-1}$, we have for all $x_{pro}$
    \begin{equation}
        \hat{v}\sigma(W x_{pro}) =  \hat{v}\sigma(\hat{W} A^{-1} A x) = \hat{v}\sigma(\hat{W} x) = y
    \end{equation}
    Therefore, the parameters $\hat{v},W$ achieves $\mathcal{L}(\hat{v}, W) = 0$.

    Following is a counterexample showing that last-layer prompts are impossible. Since $\sigma$ is the element-wise ReLU function, $\hat{W}z$ has only positive entries for all $z$. Let $B = -I$, which is an invertible diagonal matrix full of -1. Then for any $v$, we have $v \sigma(\hat{W}z_{pro}) = v\sigma(-\hat{W}x) = 0$, so the expected loss is positive. Therefore, $min_{v}\mathcal{L}(v, \hat{W}) > 0$.
\end{proof}

\begin{table*}[t]
\centering
\begin{tabular}{llcccccccc}
\toprule[1.5pt]
\multicolumn{1}{c}{Backbone} & \multicolumn{1}{l}{Method} & CIFAR10 & CIFAR100 & \multicolumn{1}{l}{ImageNet100} & Flower & Dogs & Cars & Clothing & Params \\ \midrule
\multirow{7}{*}{ResNet-50} & Retraining & 0.8281 & 0.5178 & 0.7088 & 0.8649 & 0.7942 & 0.6898 & 0.7649 & 23.529 \\
 & Head-tuning & 0.7627 & 0.4738 & 0.6167 & 0.8531 & 0.7713 & 0.6257 & 0.7324 & 0.277 \\
 & Fine-tuning & 0.8564 & 0.5271 & 0.7194 & 0.8942 & 0.8126 & 0.7548 & 0.7852 & 23.529 \\
 & Adapter & 0.8375 & 0.6021 & 0.7185 & 0.8702 & 0.8279 & 0.6816 & 0.7633 & 0.673 \\
 & Bias & 0.8319 & 0.5965 & 0.7003 & 0.8694 & \textbf{0.8316} & 0.7371 & 0.7803 & 0.494 \\
 & VPT & 0.8547 & \textbf{0.6289} & 0.7257 & 0.8876 & 0.8147 & 0.7459 & 0.7870 & 0.812 \\
 & \method{} & \textbf{0.8628} & 0.5484 & \textbf{0.7372} & \textbf{0.8976} & 0.8269 & \textbf{0.7642} & \textbf{0.7892} & 0.097 \\ \hline
\multirow{7}{*}{ResNet-101} & Retraining & 0.8357 & 0.5365 & 0.7275 & 0.8691 & 0.8015 & 0.7014 & 0.7892 & 43.713 \\
 & Head-tuning & 0.7689 & 0.4941 & 0.6287 & 0.8654 & 0.7786 & 0.6386 & 0.7641 & 0.461 \\
 & Fine-tuning & 0.8745 & 0.5487 & 0.7469 & 0.8989 & 0.8168 & 0.7715 & 0.8119 & 43.713 \\
 & Adapter & 0.8456 & 0.6233 & 0.7382 & 0.8745 & 0.8325 & 0.6895 & 0.7893 & 0.684 \\
 & Bias & 0.8517 & 0.6148 & 0.7249 & 0.8721 & \textbf{0.8421} & 0.7598 & 0.7894 & 0.517 \\
 & VPT & 0.8723 & \textbf{0.6319} & 0.7470 & 0.8898 & 0.8198 & 0.7581 & 0.8092 & 0.838 \\
 & \method{} & \textbf{0.8830} & 0.5898 & \textbf{0.7492} & \textbf{0.9024} & 0.8311 & \textbf{0.7762} & \textbf{0.8165} & 0.097 \\ \bottomrule[1.5pt]
\end{tabular}
\caption{The performance of \method{} and existing tuning baselines on six classification tasks using CNN-based pre-trained models. Params represents the maximum number of parameters that can be trained. The unit of measurement for Params is M.}
\label{tab:cnn_based}
\end{table*}

\subsection{Complexity Analysis}
We also compare our complexity with several baselines to demonstrate our superiority. For example, in the case of the ViT model, there are $N^2$ visual tokens as the input, each with a dimension of $d$. We first examine MLP-based methods, such as Adapter and Bias, which require $2d\tilde{d} \cdot L$ additional trainable parameters in $L$ layers for the projection from dimension $d$ to $\tilde{d}$. Next, for VPT, it requires $nd$ additional parameters in each layer due to the insertion of $n$ prompts, which needs $nLd$ trainable parameters in total. Our \method{} only add two prompt blocks with $m\tilde{d}$ parameters ($m \ll d$). 
In practice, our proposed \method{} has demonstrated significant advantages over Adapter and VPT, even with only slightly fewer parameters during the training stage in the below. 
This is due to the unique way in which \method{} utilizes the input data, which allows it to make more efficient use of the available parameters and achieve better results. Additionally, \method{} has a more flexible architecture that allows it to adapt to different types of input data, making it more versatile and applicable to a wide range of tasks. 

\begin{table*}[htbp]
\centering
\begin{tabular}{@{}llcccccccc@{}}
\toprule[1.5pt]
Backbone & \multicolumn{1}{l}{Method} & CIFAR10 & CIFAR100 & ImageNet100 & Flower & Dogs & Cars & Clothing & Params \\ \midrule
\multirow{7}{*}{ViT-B} & Retraining & 0.8761 & 0.5592 & 0.7277 & 0.8735 & 0.8145 & 0.7165 & 0.7959 & 85.721 \\
 & Head-tuning & 0.7914 & 0.5124 & 0.6431 & 0.8617 & 0.8019 & 0.6522 & 0.7671 & 0.187 \\
 & Fine-tuning & 0.9035 & 0.6499 & 0.7544 & 0.9003 & 0.8298 & 0.7934 & 0.8356 & 85.721 \\
 & Adapter & 0.8612 & 0.6319 & 0.7411 & 0.8777 & 0.8317 & 0.6934 & 0.8229 & 0.372 \\
 & Bias & 0.8898 & 0.6109 & 0.7326 & 0.8709 & 0.8348 & 0.7295 & 0.8210 & 0.215 \\
 & VPT & 0.9049 & \textbf{0.6689} & 0.7596 & 0.9013 & \textbf{0.8367} & 0.7682 & 0.8378 & 0.523 \\
 & \method{} & \textbf{0.9077} & 0.6541 & \textbf{0.7612} & \textbf{0.9054} & 0.8361 & \textbf{0.7991} & \textbf{0.8397} & 0.124 \\ \midrule
\multirow{7}{*}{Swin-B} & Retraining & 0.8896 & 0.5730 & 0.7316 & 0.8794 & 0.8357 & 0.7233 & 0.8112 & 86.954 \\
 & Head-tuning & 0.7991 & 0.5265 & 0.6558 & 0.8775 & 0.8150 & 0.6614 & 0.7739 & 0.295 \\
 & Fine-tuning & 0.9166 & 0.6631 & 0.7710 & 0.9056 & 0.8359 & 0.8016 & 0.8398 & 86.954 \\
 & Adapter & 0.8795 & 0.6511 & 0.7498 & 0.8812 & 0.8341 & 0.6952 & 0.8277 & 0.331 \\
 & Bias & 0.8971 & 0.6118 & 0.7401 & 0.8749 & 0.8442 & 0.7567 & 0.8261 & 0.287 \\
 & VPT & 0.9132 & \textbf{0.6816} & \textbf{0.7781} & 0.9026 & 0.8393 & 0.7982 & \textbf{0.8434} & 0.686 \\
 & \method{} & \textbf{0.9189} & 0.6705 & 0.7769 & \textbf{0.9061} & \textbf{0.8455} & \textbf{0.8027} & 0.8431 & 0.242 \\ \bottomrule[1.5pt]
\end{tabular}
\caption{The performance of \method{} and existing tuning baselines on six classification tasks using Transformer-based pre-trained models. Params represents the maximum number of parameters that can be trained. The unit of measurement for Params is M.}
\label{tab:transformer}
\end{table*}

\section{Experiment}

\subsection{Experimental Setting}
 
\noindent\textbf{Dataset.} \textbf{CIFAR10}~\cite{krizhevsky2009learning} is a dataset of 60,000 color images in 10 classes, with 6,000 images per class. There are 50,000 training images and 10,000 test images.
\textbf{CIFAR100}~\cite{krizhevsky2009learning} is a dataset of the same size as CIFAR10, but it has 100 classes containing 600 images each. There are 500 training images and 100 testing images per class.
\textbf{ImageNet100}~\cite{deng2009imagenet} is a subset of the ImageNet dataset containing 100 classes of natural images. Each class has between 500 and 1000 images for training and 50 to 100 for testing.
\textbf{Flower}~\cite{nilsback2008automated} is a dataset of images of flowers from 5 different species. It contains 4242 images, with 80-90 images per class. 
\textbf{Stanford Dogs}~\cite{khosla2011novel} is a dataset of images of 120 breeds of dogs, with a total of 20,580 images. 
\textbf{Stanford Cars}~\cite{gebru2017fine} is a dataset of cars, with a total of 16,185 images of 196 classes of cars. The dataset is organized by make, model, and year.
\textbf{Clothing}~\cite{tanaka2018joint} is a dataset containing images of various clothing types.

\noindent\textbf{Baselines.}
We compare \method{} to several commonly used protocols, including: 1) Retraining trains the entire vision model from scratch;
2) Head Fine-tuning fine-tunes the last layers of a pre-trained model while freezing the remaining layers and retraining the head classifier;
3) Fine-tuning adjusts the weights of a pre-trained model and retrains the head classifier;
4) Adapter~\cite{houlsby2019parameter} adds a new adapter structure to the transformer and updates only its parameters;
5) Bias~\cite{zaken2021bitfit} updates only the bias terms of the parameters;
6) VPT~\cite{jia2022visual} fine-tunes the model by incorporating prompts as input tokens.

\noindent\textbf{Implementation Details.}
We use the model pre-trained on ImageNet as the initialization for the following tuning for a fair comparison. Additionally, we extend our method to include CNN-based (ResNet-50, ResNet-101~\cite{he2016deep}) and Transformer-based (ViT~\cite{dosovitskiy2020image}, Swin Transformer~\cite{liu2022video}) backbones. The implementation of baselines for additional backbones involves leveraging the core idea presented in the original paper, and adapting it to suit the specific capabilities of each new backbone architecture. In experiments on the datasets above, we utilize the Adam optimizer with a momentum of 0.9, batch size of 64, and learning rate of 1e-5. The whole experiments are implemented on the NVIDIA V100 GPU with PyTorch.

\begin{table*}[htbp]
\centering
\scalebox{1}{
\setlength{\tabcolsep}{1.3mm}{
\begin{tabular}{ll|cccc|cccc|c}
\toprule[1.5pt]
\multirow{2}{*}{Backbone} & \multirow{2}{*}{Method} & \multicolumn{2}{c}{CIFAR10-LongTail} & \multicolumn{2}{c|}{CIFAR100-LongTail} & \multicolumn{4}{c|}{Fine-Grained Few-Shot} & \multirow{2}{*}{Params} \\  \cmidrule(lr){3-4}\cmidrule(lr){5-6}\cmidrule(lr){7-10}
 &  & IR100 & IR50 & IR100 & IR50 & Pets & Food-101 & Cars & Flower &  \\ \hline
\multirow{4}{*}{ResNet-50} & Head-tuning & 0.7136 & 0.7358 & 0.3569 & 0.3891 & 0.6881 & 0.6013 & 0.3846 & 0.7218 & 0.277 \\
 & Fine-tuning & 0.7638 & 0.7962 & 0.4157 & 0.4468 & 0.7340 & 0.6531 & 0.4184 & 0.7735 & 23.529 \\
 & VPT & 0.7721 & 0.7956 & 0.4192 & 0.4510 & 0.7385 & 0.6522 & 0.4189 & 0.7790 & 0.812 \\
 & \method{} & \textbf{0.7831} & \textbf{0.8114} & \textbf{0.4328} & \textbf{0.4699} & \textbf{0.7412} & \textbf{0.6584} & \textbf{0.4203} & \textbf{0.7814} & 0.097 \\ \hline
\multirow{4}{*}{ResNet-101} & Head-tuning & 0.7268 & 0.7496 & 0.3751 & 0.4036 & 0.7027 & 0.6159 & 0.4008 & 0.7542 & 0.461 \\
 & Fine-tuning & 0.7754 & 0.8083 & 0.4309 & 0.4681 & 0.7593 & 0.6694 & 0.4325 & 0.8058 & 43.713 \\
 & VPT & 0.7881 & 0.8068 & 0.4356 & 0.4710 & 0.7603 & 0.6715 & 0.4351 & 0.8134 & 0.838 \\
 & \method{} & \textbf{0.8021} & \textbf{0.8294} & \textbf{0.4662} & \textbf{0.4982} & \textbf{0.7635} & \textbf{0.6783} & \textbf{0.4470} & \textbf{0.8175} & 0.097 \\ \hline
\multirow{4}{*}{ViT-B} & Head-tuning & 0.6849 & 0.7304 & 0.3871 & 0.4315 & 0.7245 & 0.6447 & 0.4139 & 0.7691 & 0.187 \\
 & Fine-tuning & 0.7692 & 0.7834 & 0.4778 & 0.5243 & 0.7719 & 0.6834 & 0.4498 & 0.8217 & 85.721 \\
 & VPT & 0.7631 & 0.7806 & 0.4785 & 0.5254 & 0.7731 & 0.6894 & 0.4524 & 0.8412 & 0.523 \\
 & \method{} & \textbf{0.7829} & \textbf{0.8008} & \textbf{0.5024} & \textbf{0.5517} & \textbf{0.7822} & \textbf{0.6911} & \textbf{0.4588} & \textbf{0.8507} & 0.124 \\ \hline
\multirow{4}{*}{Swin-B} & Head-tuning & 0.6947 & 0.7485 & 0.4118 & 0.4568 & 0.7324 & 0.6529 & 0.4309 & 0.7814 & 0.295 \\
 & Fine-tuning & 0.7852 & 0.7995 & 0.4826 & 0.5491 & 0.7786 & 0.6764 & 0.4625 & 0.8386 & 86.954 \\
 & VPT & 0.7725 & 0.8011 & 0.4910 & 0.5482 & 0.7797 & \textbf{0.6979} & 0.4692 & 0.8497 & 0.686 \\
 & \method{} & \textbf{0.8006} & \textbf{0.8231} & \textbf{0.5276} & \textbf{0.5708} & \textbf{0.7901} & 0.6976 & \textbf{0.4773} & \textbf{0.8526} & 0.242 \\ \bottomrule[1.5pt]
\end{tabular}}}
\caption{Results of extensive experiments on the long-tailed, and few-shot datasets to validate the ability against class imbalance and sample scarcity. IR represents the imbalance ratio and the unit of measurement for Params is M.}
\label{tab:longtail}
\end{table*}

\subsection{Performances}
\noindent\textbf{Image Classification.}
Quantitative results can be seen in Table \ref{tab:cnn_based} and Table \ref{tab:transformer}.
It can be observed that the proposed \method{} performs the best overall on all six tasks when using all four models as the backbones. The best performance for each task is highlighted in bold. For example, on CIFAR100, \method{} achieves an accuracy of 54.84 $\%$ when using ResNet-50, and 58.98 $\%$ when using ResNet-101, the highest among all the methods. For the transformer-based (i.e., ViT-B and Swin-B) methods, \method{} achieves accuracy improvement of 1.36 $\%$, 0.68 $\%$, 0.51 $\%$, 3.31 $\%$, 0.38 $\%$ and 0.19 $\%$ over other tuning methods, respectively on CIFAR10, CIFAR100, ImageNet100, Flower, Dogs, and Cars datasets. Regarding the maximum number of trainable parameters, the proposed method has the lowest value among all the methods, with only 0.097 M trainable parameters.
It is also worth noting that the performance of the other methods varies depending on the task and the backbone model used. For example, while VPT performs well on CIFAR100, it is not as effective on other tasks. On the other hand, the adapter method performs relatively well on the Flower and Dogs tasks, but not as well on the other tasks.
In summary, the proposed \method{} is the most effective tuning method across all tasks and backbone models, with the advantage of having the lowest number of trainable parameters.

\noindent\textbf{Long-tail Class Distribution.}
We perform experiments on benchmark datasets that have a long-tail class distribution, such as CIFAR10-LongTail and CIFAR100-LongTail. The results of the imbalance ratio 50 and 100 are shown in Table \ref{tab:longtail}.
Our proposed \method{} algorithm outperforms the best baseline, VPT, under all the settings. Specifically, we observe a gain of approximately 4\% in validation accuracy, while reducing the trainable parameters by 88\%. This trend holds across other settings as well.
Additionally, when compared with VPT on long-tailed CIFAR-10 with an imbalanced ratio of 100 under ViT-B, \method{} achieves superior validation accuracy using only 4.2x fewer trainable parameters. When evaluated under Swin-B, \method{} outperforms VPT on long-tailed CIFAR-10 with imbalanced ratios of 50 and 100, while also reducing the trainable parameters by 64.7\%.

\noindent\textbf{Few-shot Learning.}
We perform experiments on benchmark datasets, such as Pets~\cite{parkhi2012cats}, Food-10~\cite{bossard2014food}, Cars~\cite{gebru2017fine}, and Flower~\cite{nilsback2008automated} with eight examples per class,  which are widely used for evaluating few-shot learning algorithms.
Our experimental results, shown in Table \ref{tab:longtail}, demonstrate that \method{} achieves state-of-the-art results on average, while using the fewest trainable parameters. It is worth noting that although VPT outperforms our method on some datasets for the image classification task, we still achieve the best performance on all datasets in this scenario, which highlights our method's advantage. These observations confirm the capability and efficiency of our method in the low-data regime and further verify the effectiveness of the lightweight implicit vision prompt design.

\begin{table}[htbp]
\centering
\scalebox{1}{
\begin{tabular}{llccc}
\toprule[1.5pt]
 & Method & CIFAR10 & CIFAR100 & INet100 \\ \hline
\multirow{4}{*}{\rotatebox{90}{ResNet}} & \method{}-P1 & 0.8501 & 0.5349 & 0.7146 \\
 & \method{}-P2 & 0.8239 & 0.5136 & 0.6989 \\
 & \method{}-R & 0.8516 & 0.5356 & 0.7197 \\
 & \method{} & 0.8628 & 0.5484 & 0.7372 \\ \hline
\multirow{4}{*}{\rotatebox{90}{ViT-B}} & \method{}-P1 & 0.8876 & 0.6413 & 0.7492  \\
 & \method{}-P2 & 0.8527 & 0.6304 & 0.7186 \\
 & \method{}-R & 0.8896 & 0.6309 & 0.7437 \\
 & \method{} & 0.9077 & 0.6541 & 0.7612 \\ \bottomrule[1.5pt]
\end{tabular}}
\caption{Ablation study on important inner modules under two different backbones. INet100 represents ImageNet100.}
\label{tab:ablation}
\end{table}

\subsection{Ablation Study}
With the ImageNet pre-trained ResNet-50 and Vit-B, we perform extensive ablation studies to analyze the developed \method{} systematically. We introduce three model variants as follows:
(1) \textbf{\method{}-P1} removes the equilibrium layer before the pre-trained backbone to activate low-level features, i.e., $\mathcal{P}_{1}$ in the Eq. \ref{eq:p1}.
(2) \textbf{\method{}-P2} removes the equilibrium layer behind the pre-trained backbone to activate high-level features, i.e., $\mathcal{P}_{2}$ in the Eq. \ref{eq:p2}.
(3) \textbf{\method{}-R} removes the robust training mechanism with the standard optimization for all the parameters, i.e., optimization in Eq. \ref{eq:phi_c} and Eq. \ref{eq:phi_n}.
The results of these model variants are summarized in Table \ref{tab:ablation}. We have the following observations. First, our \method{} outperforms \method{}-$\mathcal{P}_{1}$ and \method{}-$\mathcal{P}_{2}$, which indicates that implicit vision prompt blocks work for vision semantic information activation. 
Second, \method{}-$\mathcal{P}_{2}$ obtains much worse than \method{}, showing that the high-level activation is vital for the vision prompt tuning.  
Third, the robust training mechanism allows better network optimization, as shown by \method{}-R's lower performance of 2\% compared to \method{}, validating the superiority of robust training for better optimization.

\subsection{Sensitivity Analysis}
In Figure \ref{fig:sensi}, we investigate the sensitivity of two hyper-parameters: $\tau$ in the robust training mechanism, which separates the crucial and non-crucial parameters in Eq. \ref{eq:phi_c} and Eq. \ref{eq:phi_n}. Moreover, we also investigate the number of the layer in DEQ by stacking the single layer attracted by its striking performances. 
We first vary $\tau$ in \{0.2, 0.4, 0.6, 0.8\} with other parameters fixed. 
As $\tau$ rises, the performance first increases and then decreases a little. The potential reason is that too large $\tau$ could filter the essential parameters for optimization.
We can observe that the performance of ours is not sensitive to $\tau$ in the range of $[0.4,0.6]$, and we can set it to any value in that interval. Further, we changed the number of layers from 1 to 4 with other parameters fixed. Obviously, our method can achieve a slight performance gain with the layer ranging from $1$ to $4$ but at the cost of several times the computation time. Therefore, $\tau$ and the layer number are set to $0.4$ and $1$ as default, respectively.

\begin{figure}[t]
    \centering
    \includegraphics[width=0.43\textwidth]{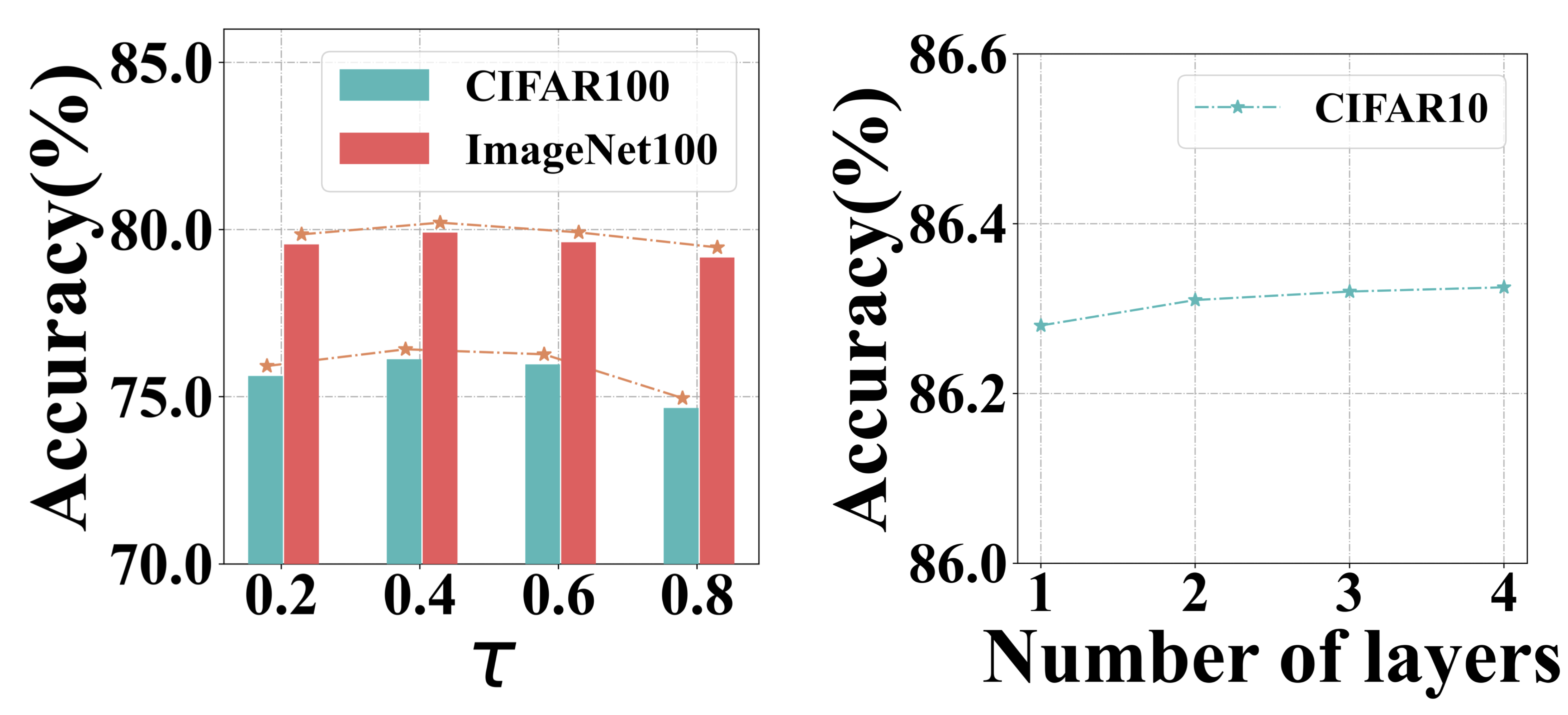}
    \caption{Sensitivity analysis on two hyper-parameters.}
    \label{fig:sensi}
\end{figure}

\section{Conclusion}
In conclusion, we propose an efficient vision tuning method named \method{} that addresses the heavy computational costs. By drawing inspiration from deep implicit models with stable memory costs, \method{} only requires two equilibrium implicit layers in two ends of the pre-trained main backbone with parameters frozen. Additionally, pruning the parameters in these two layers according to the lottery hypothesis reduces training parameters. \method{} can obtain higher performance with a smaller parameter size than the state-of-the-art baseline VPT, especially under challenging scenes. Our experiments demonstrate that \method{} has a good generalization performance, making it an easy way to boost applications in the future. Overall, \method{} provides an economical solution for vision tasks and is promising for a wide range of datasets.

\section*{Acknowledgements}
Zhouchen Lin was supported by the National Key R\&D Program of China (2022ZD0160302), the NSF China (No. 62276004), and the major key project of PCL, China (No. PCL2021A12).

\bibliography{aaai24}

\end{document}